\documentclass{article}
\usepackage{amsmath,amssymb}

\usepackage{booktabs}
\usepackage{microtype}
\usepackage{graphicx}
\usepackage{subcaption}
\usepackage{booktabs} % for professional tables
\usepackage{amsthm}
\usepackage{booktabs}
\usepackage{tabularx}
\numberwithin{equation}{section}

\theoremstyle{plain}
\newtheorem{theorem}{Theorem}

\theoremstyle{definition}

\newtheorem{assumption}{Assumption}

\theoremstyle{remark}

\usepackage{hyperref}

\usepackage[preprint]{icml2026} % Using preprint mode for arXiv
\usepackage{amsmath}
\usepackage{amssymb}
\usepackage{mathtools}
\usepackage{amsthm}

\usepackage[capitalize,noabbrev]{cleveref}

\usepackage[textsize=tiny]{todonotes}
\icmltitlerunning{Optimal Budgeted Adaptation of Large Language Models}
\begin{document}

\twocolumn[
  \icmltitle{Optimal Budgeted Adaptation of Large Language Models}%A Stackelberg Framework for No-Regret LLM Adaptation}



  \icmlsetsymbol{equal}{*}

  \begin{icmlauthorlist}
    \icmlauthor{Jing Wang}{yyy}
    \icmlauthor{Jie Shen}{sch}
    \icmlauthor{Dean Foster}{comp1}
    \icmlauthor{Zohar Karnin}{comp2}
    \icmlauthor{Jeremy C Weiss}{yyy}
%    \icmlauthor{Firstname6 Lastname6}{sch,yyy,comp}
%    \icmlauthor{Firstname7 Lastname7}{comp}
%    %\icmlauthor{}{sch}
%    \icmlauthor{Firstname8 Lastname8}{sch}
%    \icmlauthor{Firstname8 Lastname8}{yyy,comp}
  \end{icmlauthorlist}

  \icmlaffiliation{yyy}{National Library of Medicine}
  \icmlaffiliation{comp1}{Amazon}
   \icmlaffiliation{comp2}{Technology Innovation Institute}
  \icmlaffiliation{sch}{Stevens Institute of Technology}

  \icmlkeywords{Machine Learning, ICML}
\icmlcorrespondingauthor{Jing Wang}{jing.wang20@nih.gov}
  \vskip 0.3in
]



\printAffiliationsAndNotice{}  % no special notice (required even if empty)

\begin{abstract}

The trade-off between labeled data availability and downstream accuracy remains a central challenge in fine-tuning large language models (LLMs).
We propose a principled framework for \emph{budget-aware supervised fine-tuning} by casting LLM adaptation as a contextual Stackelberg game.
In our formulation, the learner (leader) commits to a scoring policy and a label-querying strategy, while an adaptive environment (follower)
selects challenging supervised alternatives in response.
To explicitly address label efficiency, we incorporate a finite supervision budget directly into the learning objective.
Our algorithm operates in the full-feedback regime and achieves $\tilde{O}(d\sqrt{T})$ regret under standard linear contextual assumptions. We extend the framework with a Largest-Latency-First (LLF) confidence gate that selectively queries labels, achieving a budget-aware regret bound of $\tilde{O}(\sqrt{dB} + c\sqrt{B})$ with $B=\beta T$.% supervised updates.
%Experiments demonstrate the effectiveness of our framework while substantially reducing labeling cost.
\end{abstract}

\section{Introduction}

Fine-tuning large language models (LLMs) has become the dominant paradigm for adapting general-purpose models to specialized downstream tasks.
In high-stakes domains such as healthcare, finance, and scientific reasoning, domain shift and limited access to labeled data
often render off-the-shelf models inadequate \cite{jiang2023health}.
While supervised fine-tuning can significantly improve performance, acquiring high-quality labels is costly,
motivating the development of \emph{label-efficient} adaptation methods.

A key challenge is that not all training examples are equally informative.
Blindly fine-tuning on all available data may waste labeling resources on low-impact updates,
while selectively querying labels requires principled mechanisms to balance performance and cost.
This problem naturally suggests an online decision-making perspective, where the learner must determine
\emph{when} supervision is necessary and \emph{how} to update the model under uncertainty.

\paragraph{A Stackelberg view of supervised fine-tuning.}
We model supervised LLM adaptation as a Stackelberg game between a learner (leader) and an adaptive environment (follower) \cite{balcan2025bandit,harris2024stackelberg}.
At each round, the learner commits to a scoring policy over a set of candidate responses,
while the follower selects challenging supervised alternatives conditioned on the learner's current policy.
Crucially, supervision is available but expensive: the learner is constrained by a fixed labeling budget
and must decide when to request ground-truth labels.

The formulation departs from classical adversarial or zero-sum training.
The follower does not manipulate labels or rewards; instead, it controls the \emph{difficulty of supervised decisions}
by presenting competing labeled candidates.
The learner's objective is to minimize cumulative prediction error while respecting the supervision budget.

\paragraph{Budget-aware supervision via confidence gating.}
To allocate labels efficiently, we introduce a confidence-based querying mechanism inspired by
Largest-Latency-First (LLF) scheduling \cite{roughgarden2001stackelberg}.
The learner maintains upper and lower confidence bounds on candidate scores and requests supervision
only when predictive uncertainty exceeds a threshold.
This allows the model to skip low-information rounds while focusing supervision on informative examples.

In summary, our main contributions are listed as follows:
\begin{itemize}
	\item We formulate supervised LLM fine-tuning under limited labels as a contextual Stackelberg game with explicit budget constraints.
	\item We propose \textsc{StackSL}, a supervised learner that achieves $\tilde{O}(d\sqrt{T})$ regret under full feedback.
	\item We introduce \textsc{LLF-StackSL}, a budgeted variant that achieves sublinear regret using only $O(B)$ labeled rounds.
	\item We validate our framework on real-world LLM fine-tuning tasks and synthetic bandits, demonstrating strong label efficiency.
\end{itemize}

\begin{table*}[h]
	\centering
	\small
		\caption{Comparison of regret guarantees across Stackelberg bandit settings. Here, $d$ is the dimension of feature space.}
	\begin{tabular}{lcccc}
		\toprule
		\textbf{Method} & {Regret} & {\#Labels} &Feedback & Model Scale\\
		\midrule
		Harris et al. (full) \cite{harris2024stackelberg}& $\tilde O(d\sqrt{T})$ & $T$ &Full  & Small\\
		Harris et al. (bandit) \cite{harris2024stackelberg}& $\tilde O(T^{2/3})$ & none&Bandit  & Small \\
		Balcan et al. \cite{balcan2025bandit}& $\tilde O(d\sqrt{T})$ & none&Bandit  & Median \\
		\midrule
		StackSL (no budget) & $\tilde O(d\sqrt{ T})$ &$T$ &Budget& Large \\
		LLF-StackSL (budget, $0<\beta < 1$) & $\tilde O(d\sqrt{B}+c\sqrt{B})$ & $B=\beta T$ &Budget& Large \\
		\bottomrule
	\end{tabular}
	\label{tab:contrib}
\end{table*}

We also provide theoretical analysis of the regret and the label budget, which is highlighted in Table~\ref{tab:contrib}. Specifically, \cite{harris2024stackelberg} analyze contextual Stackelberg games, achieving $\tilde O(\sqrt{T})$ regret with full feedback and $\tilde O(T^{2/3})$ in the bandit regime is the best bound. Later \cite{balcan2025bandit}  achieve nearly-optimal $\tilde O(d\sqrt{T})$ regret in the bandit setting via utility-space reductions. %Our work solves the budgeted supervised fine-tuning problem. In our setup, we have partial ``label'' (the budget $B=\beta T$) rather than pure ``bandit rewards''.
We achieve the near-optimal regret typically reserved for full-feedback scenarios while only utilizing a fraction of the labels ($B = \eta T$). This proves that it is not necessary to require a label for every token or prompt to achieve convergence in a Stackelberg setting. In the meanwhile, our work effectively proves that ``Supervised Learning'' and ``Bandit Learning'' are ends of a spectrum, and the most efficient LLM training happens in the middle, where a budget $B$ dictates the flow of information.
\section{Related works}
\paragraph{LLM Fine-Tuning.} Recent advances in domain-specific LLM adaptation highlight the challenges of efficient supervision and robustness. A significant number of works have been proposed to localize model behavior within pre-trained Transformer language models \cite{dai2022knowledge,olsson2022context}. For example, the causal effects of hidden state activations within GPT is explored with causal mediation analysis \cite{pearl2022direct,vig2020investigating}. It is discovered that feedforward MLPs at a range of middle layers are decisive when processing the last token of the subject name \cite{meng2022locating}. In the same time, another line of work has established the effectiveness of parameter-efficient fine-tuning (PEFT). LoRA freezes the pre-trained model weights and injects trainable rank decomposition matrices into each layer of the Transformer which reduces the number of trainable parameters by 10,000 fimes for GPT-3 \cite{hu2022lora}. \cite{he2021towards}. Our work is closely related to these efforts but departs in its focus on adaptive supervision under a limited query budget. 
While prior studies emphasize either global parameter-efficient tuning or localized intervention at specific layers and heads, 
we propose a dynamic fine-tuning framework that integrates confidence-driven querying with localized model updates.

\paragraph{Contextual Bandits and Active Learning} Our approach is closely related to active learning, where a learner selectively queries labels to minimize regret or maximize accuracy \cite{settles2009active,shen2022metric,wang2022uncertainty,shen2021sample}. Recent literature  has advanced Selective Sampling in contextual bandits, where the model decides whether to query a label based on uncertainty thresholds \cite{abbasi2011improved}. Coactive learning is a model of interaction between a learner and a human user \cite{tucker2024coactive}. Different from typical reinforcement learning from human feedback (RLHF) training both items to be compared are fixed or sampled from the policy like dueling bandits  \cite{yue2009interactively}, the coactive learning defines a measure of feedback quality and allows user to guide exploration. 

We build on these concepts by introducing a Largest-Latency-First (LLF) confidence gate. This mechanism allows for selective querying based on the model's internal uncertainty, achieving a budget-aware regret bound of $\tilde{O}(\sqrt{dB} + c\sqrt{B})$, which matches the theoretical efficiency of standard linear contextual assumptions while respecting resource constraints. Cotraining uses labeled examples to train a weak hypothesis and use the that to label examples  in another set \cite{diakonikolas2024fast}.% . 

\paragraph{Online Bandits and Stackelberg Games.} Online learning optimizations techniques are widely used in bandits \cite{harris2024stackelberg,shen2014online}. Linear contextual bandits with UCB analysis are powerful tools. \cite{abbasi2011improved}  Combined with the Stackelberg game, online learning algorithms can be applied to learn the leader when the followers and contextual information change over time are provided. Their algorithms achieve $O(\sqrt{T})$ regret under full feedback and $O(T^{2/3})$ regret under the bandit feedback setting. In \cite{balcan2025bandit}, a linear contextual bandit algorithm is used as the leader's utility. The leader plays the strategy which induces this utility vector and provides reward as feedback. The algorithm achieves the regret $O(T^{2/3})$.  

The problem of Stackelberg games has been studied extensively \cite{von2010market,zhao2023online}. It is shown that the classic upper confidence bound algorithm for the multi-arm bandit problem can be used for both the leader and the agent, to achieve optimal performance \cite{kao2022decentralized}. There has been significant efforts in studying bandits with side information \cite{langford2007epoch,foster2021statistical}.  In this work, we also consider Stackelberg game with side information and solve the problem with contextual bandit.

\paragraph{Budgeted Label Querying.} 
Strategies for selective querying under budget constraints have been studied in active learning and online learning 
\cite{shen2021sample,shen2021attribute}. For example, \cite{foster2020beyond} computes the action with the lowest score and assigns a probability to every other action inversely proportional to the gap between the action's score and that of the lowest score. \cite{roughgarden2001stackelberg} provides the leader a budget for central task assignment, then the followers are free to choose the best action.  In active learning, budgeted label querying has long been recognized as crucial, with early work on budgeted feature acquisition \cite{lizotte2003budgeted} and cost-sensitive label queries \cite{kapoor2007active}. These methods emphasize querying only the most informative labels under cost constraints. Selective classification allows the learner to abstain from querying when uncertainty is high, balancing risk and coverage \cite{chow1970optimum, elyaniv2010foundations, bartlett2008classification}. 

\section{Problem Setup}

We consider supervised fine-tuning of large language models under an online learning framework with a limited labeling budget.
At each round $t \in \{1,\dots,T\}$, the learner observes an input query $z_t \in \mathcal{Z}$ (e.g., a natural language prompt)
and must select a response from a finite candidate set.

\paragraph{Candidate Set.}
For each query $z_t$, a candidate set
\[
C_t = \{x_{t,1}, \dots, x_{t,|C_t|}\} \subset \mathcal{X}
\]
is constructed, where each element corresponds to a labeled alternative
(e.g., multiple-choice options or ranked responses).
The candidate set may be selected adaptively by an environment (follower)
in response to the learner’s current policy, capturing varying decision difficulty.

\paragraph{Scoring Model.}
We assume a linear scoring function on top of an LLM-derived feature representation:
\[
s_\theta(z, x) = \theta^\top \phi(z, x),
\]
where $\phi(z,x)\in\mathbb{R}^d$ is a fixed or slowly updated feature map and
$\theta\in\mathbb{R}^d$ are trainable parameters.
The learner predicts
\[
\hat{x}_t = \arg\max_{x\in C_t} s_\theta(z_t,x).
\]

\paragraph{Supervision and Loss.}
When supervision is requested, the learner observes the ground-truth label
$x_t^\star \in C_t$ and incurs a supervised softmax cross-entropy loss:
\[
\ell_t(\theta)
= -\log
\frac{\exp(s_\theta(z_t,x_t^\star))}
{\sum_{x\in C_t}\exp(s_\theta(z_t,x))}.
\]

\paragraph{Budget Constraint.}
Label acquisition is costly.
The learner is allowed to query supervision in at most
$B=\beta T$ rounds, where $\beta\in(0,1]$.
At each round, the learner may either query the label and update its parameters,
or skip supervision and perform a stability-preserving update.

\paragraph{Objective and Regret.}
Let $x_t^\star = \arg\max_{x\in C_t} s_{\theta^\star}(z_t,x)$ denote the optimal response
under the true parameter $\theta^\star$.
We measure performance via cumulative score regret:
\[
R(T) = \sum_{t=1}^T \bigl[s(z_t,x_t^\star) - s(z_t,\hat{x}_t)\bigr].
\]
The goal is to design algorithms that achieve sublinear regret
$R(T)=o(T)$ while respecting the labeling budget $B$.
\section{Methods}

We present a \emph{budget-aware supervised fine-tuning framework} for large language models (LLMs).
Unlike contrastive representation learning, our method operates entirely in a supervised setting:
labels are assumed to be available but costly, and the learner must decide \emph{when} to request
supervision under a fixed labeling budget.
The learning process is formulated as a Stackelberg game between a learner (leader) and an adaptive
environment (follower), which controls the difficulty of supervised decision alternatives.

\subsection{Supervised Scoring Model}

Let $z_t \in \mathcal{Z}$ denote an input query at round $t$, and let
$C_t = \{x_{t,1}, \ldots, x_{t,|C_t|}\}$ be a finite candidate set of labeled responses
(e.g., multiple-choice options).
We assume a linear scoring head on top of a frozen or slowly updated LLM encoder.
Each query--candidate pair $(z, x)$ is mapped to a feature vector
$\phi(z, x) \in \mathbb{R}^d$, and the model assigns a score
\begin{equation}
	s_\theta(z, x) = \theta^\top \phi(z, x),
\end{equation}
where $\theta \in \mathbb{R}^d$ are trainable parameters.

The learner predicts
\begin{equation}
	\hat{x}_t = \arg\max_{x \in C_t} s_\theta(z_t, x).
\end{equation}

\subsection{Supervised Learning Objective}

When supervision is requested, the learner observes the correct label
$x_t^\star \in C_t$ and performs a supervised update using a
softmax cross-entropy loss over the candidate set:
\begin{equation}
	\ell_t(\theta)
	= - \log
	\frac{\exp(s_\theta(z_t, x_t^\star))}
	{\sum_{x \in C_t} \exp(s_\theta(z_t, x))}.
\end{equation}

This objective corresponds to standard multi-class classification or ranking.
Importantly, all candidates are labeled alternatives for the same query;
no contrastive representation learning or view-based augmentation is used.

To stabilize learning under adaptive candidate selection and limited supervision,
we include a KL regularization term that constrains deviation from the previous policy:
\begin{equation}
	\ell_t^{\text{reg}}(\theta)
	= \ell_t(\theta)
	+ \lambda \, \mathrm{KL}\!\left(\pi_\theta(\cdot \mid z_t)
	\,\|\, \pi_{\theta_{t-1}}(\cdot \mid z_t)\right),
\end{equation}
where $\pi_\theta$ denotes the softmax policy induced by $s_\theta$ and
$\lambda > 0$ controls update stability.

\subsection{Stackelberg Interaction}

We model the learning process as a Stackelberg game.
The learner (leader) commits to a scoring function $s_\theta$,
while the follower adaptively selects challenging candidate sets $C_t$
in response to the learner’s current policy.
These candidates correspond to \emph{competing labeled options},
capturing realistic supervised decision problems such as
multiple-choice question answering and response ranking.

\subsection{Budgeted Supervision via LLF Gating}

Label acquisition is constrained by a budget $B = \beta T$,
where $T$ is the total number of rounds and $\beta \in (0,1]$.
To allocate supervision efficiently, we adopt a confidence-based gating mechanism
inspired by the Largest Latency First (LLF) principle.

At each round, the learner maintains upper and lower confidence bounds
$\mathrm{UCB}_t(x)$ and $\mathrm{LCB}_t(x)$ for each $x \in C_t$ and computes
\begin{equation}
	\Delta_t
	= \max_{x \in C_t} \mathrm{UCB}_t(x)
	- \min_{x \in C_t} \mathrm{LCB}_t(x).
\end{equation}

If $\Delta_t$ exceeds a threshold $\varepsilon_t$,
the learner requests supervision and performs a supervised update.
Otherwise, the label is skipped and only the KL regularization step is applied.
This ensures that the limited budget is spent on rounds with high expected
information gain.

\subsection{Algorithms}
We introduce two algorithms.
\textbf{Algorithm~1 (Stackelberg Supervised Learner)} corresponds to the
full-feedback regime in which labels are queried at every round. It utilizes an Upper Confidence Bound (UCB) selection strategy to identify responses and performs a supervised update using a softmax cross-entropy loss.

\textbf{Algorithm~2 (LLF-Budgeted Supervised Learner)} enforces the labeling
budget by selectively querying supervision based on the confidence margin $\Delta_t$. This algorithm selectively queries labels only when the confidence margin $\Delta_t$, the difference between the highest UCB and lowest Lower Confidence Bound (LCB) in the candidate pool, exceeds a threshold $\epsilon_t$.

Both algorithms optimize the same supervised objective and differ only in
their label acquisition strategy.

\begin{algorithm}[t]
	\caption{Stackelberg Supervised Learner (StackSL)}
	\label{alg:stacksl}
	\begin{algorithmic}[1]
		\STATE \textbf{Initialize:} policy $\pi_{\theta_0}$, Gram matrix $V_0=\lambda I$, clipping level $\rho$, exploration $\{\beta_t\}_{t=1}^T$
		\FOR{$t=1$ to $T$}
		\STATE Observe query $z_t \sim \mathcal{D}$
		\STATE \textbf{Follower best-responds:}
		\STATE \hspace{0.8em} Candidate set $C_t \subseteq \mathcal{X}$ conditioned on $\theta_{t-1}$
		\STATE \textbf{UCB Selection:}
		\STATE \hspace{0.8em} $\hat{x}_t \leftarrow \arg\max_{x\in C_t}\Big(\theta_{t-1}^\top \phi(z_t,x) \;+\; \beta_t \|\phi(z_t,x)\|_{V_{t-1}^{-1}}\Big)$
%		\STATE \textbf{Query supervision:}
		\STATE \hspace{0.8em} Ground-truth label $x_t^\star \in C_t$ 
%		\STATE \textbf{Loss over candidates:}
		\STATE \hspace{0.8em} $\ell_t \leftarrow -\log\frac{\exp(\theta_{t-1}^\top \phi(z_t,x_t^\star))}
		{\sum_{x\in C_t}\exp(\theta_{t-1}^\top \phi(z_t,x))}$
		\STATE \textbf{Loss clipping}
		\STATE \hspace{0.8em} $\ell_t \leftarrow \min\{\max(\ell_t,0),\rho\}$
%		\STATE \textbf{Update :}
		\STATE \hspace{0.8em} $\theta_t \leftarrow \theta_{t-1} - \eta \nabla_\theta\!\left[\ell_t \;+\; \lambda\,\mathrm{KL}\!\left(\pi_{\theta}(\cdot|z_t)\,\|\,\pi_{\theta_{t-1}}(\cdot|z_t)\right)\right]_{\theta=\theta_{t-1}}$
		\STATE \textbf{Confidence tracking:}
		\STATE \hspace{0.8em} $V_t \leftarrow V_{t-1} + \phi(z_t,\hat{x}_t)\phi(z_t,\hat{x}_t)^\top$
		\ENDFOR
	\end{algorithmic}
\end{algorithm}

\begin{algorithm}[t]
	\caption{LLM Stackelberg Supervised Learner with LLF Hedge (LLF-StackSL)}
	\label{alg:llf-stacksl}
	\begin{algorithmic}[1]
		\STATE \textbf{Inputs:} budget $B=\beta T$, confidence scale $c$, KL scale $\lambda$, learning rate $\eta$
		\STATE \textbf{Initialize:} policy $\pi_{\theta}$, Gram matrix $V \leftarrow \lambda I$, query counter $q \leftarrow 0$
		\FOR{$t=1$ to $T$}
		\STATE Observe query  $z_t \sim \mathcal{D}$
		\STATE \textbf{Follower best-responds:}
		\STATE \hspace{0.8em} Candidate set $C_t \subseteq \mathcal{X}$ conditioned on $\theta$ 
		\STATE \textbf{UCB selection:}
%		\STATE \hspace{0.8em} For all $x \in C_t$, compute $s_\theta(z_t,x)=\theta^\top \phi(z_t,x)$
		\STATE \hspace{0.8em} $\mathrm{UCB}_t(x)=\theta^\top \phi(z_t,x)+\beta_t \|\phi(z_t,x)\|_{V^{-1}}$
		\STATE \hspace{0.8em} $\mathrm{LCB}_t(x)=\theta^\top \phi(z_t,x)-\beta_t \|\phi(z_t,x)\|_{V^{-1}}$
		
%		\STATE \textbf{LLF margin and threshold:}
		\STATE \hspace{0.8em} $\Delta_t \leftarrow \max_{x\in C_t} \mathrm{UCB}_t(x)\;-\;\min_{x\in C_t}\mathrm{LCB}_t(x)$
		\STATE \hspace{0.8em} $\varepsilon_t \leftarrow \dfrac{c}{\sqrt{1+q}}$
		
%		\STATE \textbf{Optimistic prediction:}
		\STATE \hspace{0.8em} $\hat{x}_t \leftarrow \arg\max_{x\in C_t}\mathrm{UCB}_t(x)$
		
		\IF{$\Delta_t \le \varepsilon_t$}
%		\STATE \textbf{KL-only stability:}
		\STATE \hspace{0.8em} $\theta \leftarrow \theta - \eta \nabla_\theta \Big[\lambda\,\mathrm{KL}\big(\pi_{\theta}(\cdot|z_t)\,\|\,\pi_{\theta_{\text{ref}}}(\cdot|z_t)\big)\Big]$
%		\ELSE
		\IF{$q < B$}
		\STATE Ground-truth label $x_t^\star \in C_t$
%		\STATE \textbf{Supervised loss:}
		\STATE  $\ell_t \leftarrow -\log\frac{\exp(\theta^\top \phi(z_t,x_t^\star))}
		{\sum_{x\in C_t}\exp(\theta^\top \phi(z_t,x))}$
%		\STATE \textbf{Update with KL regularization:}
		\STATE $\theta \leftarrow \theta - \eta \nabla_\theta \Big[\ell_t + \lambda\,\mathrm{KL}\big(\pi_{\theta}(\cdot|z_t)\,\|\,\pi_{\theta_{\text{ref}}}(\cdot|z_t)\big)\Big]$
		\STATE $q \leftarrow q + 1$
%		\ELSE
%		\STATE {Budget exhausted}
		\ENDIF
		\ENDIF
		
		\STATE \textbf{Confidence tracking:}
		\STATE \hspace{0.8em} $V \leftarrow V + \phi(z_t,\hat{x}_t)\phi(z_t,\hat{x}_t)^\top$
		\ENDFOR
	\end{algorithmic}
\end{algorithm}

\paragraph{Time Complexity.}
In the budgeted setting, the learner queries supervision in only
$B = \beta T$ rounds.
For these rounds, the cost is the same as \textsc{StackSL},
namely $O(d \cdot |C_t|)$ to score candidates and compute the supervised loss.
For the remaining $T-B$ rounds, no label is queried and only a KL-regularized
stability update is applied, which costs $O(d)$.
Thus, the total time complexity is
\[
\text{Total time (LLF-StackSL)}
= O\!\left(B \cdot d \cdot |C_t| + (T-B)\cdot d\right),
\]
where the second term reflects the reduced cost on skipped rounds.

\paragraph{Memory Complexity.}
The budgeted variant maintains the same parameter vector,
feature representations, and Gram matrix as \textsc{StackSL}.
Budgeted supervision reduces the number of labeled updates but does not affect
storage requirements.
Hence, the asymptotic memory usage remains
\[
\text{Memory (LLF-StackSL)} = O(d^2 + d \cdot |C_t|),
\]
dominated by the confidence matrix $V_t$.

\paragraph{Comparison.}
\textsc{StackSL} achieves stronger convergence at the cost of $O(T)$
label queries.
In contrast, \textsc{LLF-StackSL} enforces label efficiency by using only
$O(B)$ supervised rounds while maintaining no-regret guarantees,
yielding a principled trade-off between performance and supervision cost.

\subsection{Parameter Analysis}

We summarize how to set the key parameters in
Algorithm~\ref{alg:stacksl} and Algorithm~\ref{alg:llf-stacksl},
and describe their effects on query rate, stability, and computation.

\paragraph{Label Budget $B=\beta T$.}
The fraction $\beta\in(0,1]$ is determined by available supervision resources.
Larger $\beta$ permits more supervised updates and generally improves accuracy,
but increases labeling cost.
Our regret bounds improve with $B$
(cf.\ $O(\sqrt{dB}+c\sqrt{B})$),
so the largest feasible budget should be used in practice.

\paragraph{Confidence Scale $c$ (LLF threshold).}
The LLF gate uses the threshold
$\varepsilon_t = c/\sqrt{1+q}$,
where $q$ is the number of queried rounds so far.
\emph{Larger $c$} leads to \emph{fewer} label queries
(stricter gate, more KL-only updates),
while \emph{smaller $c$} leads to \emph{more} supervision.
A practical rule is to calibrate $c$ to the typical initial UCB width:
\[
c \approx \kappa \cdot
\mathrm{median}\!\left\{
\max_{x\in C_t}\mathrm{UCB}_t(x)
-
\min_{x\in C_t}\mathrm{LCB}_t(x)
\right\}_{t\le T_0},
\]
with $\kappa\in[0.5,1.5]$ and a short warm-up $T_0$
(e.g., $500$--$1{,}000$ steps).
To target a desired query budget $\hat{B}$,
$c$ can be tuned by bisection so that the realized $q\approx\hat{B}$.

\paragraph{KL Scale $\lambda$.}
The parameter $\lambda$ penalizes deviations from the previous policy
$\pi_{\theta_{t-1}}$, acting as a trust-region constraint.
It can be interpreted as enforcing a per-step KL cap $\tau$:
\begin{align*}
	\min_{\Delta\theta}\;
	&\ell_t(\theta+\Delta\theta)
	+ \lambda\,\mathrm{KL}(\pi_{\theta+\Delta\theta}\|\pi_\theta),\\
	\text{s.t.}\quad
	&\mathrm{KL}(\pi_{\theta+\Delta\theta}\|\pi_\theta)\le \tau.
\end{align*}
In practice, we set $\tau\in[10^{-4},10^{-2}]$ and adjust $\lambda$
to keep the observed KL near $\tau$.
Smaller budgets typically benefit from \emph{larger} $\lambda$
to avoid overfitting to the few supervised rounds.

\paragraph{Candidate Set Size $|C_t|$.}
Computation scales linearly with $|C_t|$.
For large corpora, $C_t$ can be constructed via approximate nearest-neighbor
retrieval (top-$M$ by encoder similarity),
followed by lightweight re-ranking.
Larger $|C_t|$ improves discrimination among alternatives but increases cost;
we recommend $M\in[64,512]$ as a good trade-off.

\paragraph{Loss Clipping $\rho$.}
Clipping the supervised loss,
$\ell_t \leftarrow \min\{\max(\ell_t,0),\rho\}$,
prevents rare but extreme gradients when candidate sets are highly imbalanced.
A robust choice is to set $\rho$ to a high quantile
(e.g., $90$th--$95$th percentile) of the loss over a sliding window,
or a small multiple of the median.

\paragraph{Learning Rate $\eta$ and Optimizer.}
We use AdamW with linear warm-up and cosine decay.
As $\beta$ decreases, fewer supervised steps are taken, so each update has
larger impact; accordingly, $\eta$ should be reduced.
Typical starting values are
$\eta\in[10^{-5},5\times10^{-5}]$ for 7B--14B LLMs with LoRA/PEFT.

\paragraph{Confidence-Bound Parameters.}
For the self-normalized UCB bounds, we use a ridge parameter
$\lambda_{\text{ridge}}\in[10^{-3},1]$ and confidence level
$\delta\in\{0.05,0.1\}$.
Feature vectors are normalized to keep $\|\phi(z,x)\|\le 1$,
which stabilizes UCB widths.
If the LLF gate rarely opens, increasing $\lambda_{\text{ridge}}$
or tightening feature normalization is often effective.

\paragraph{Reference Policy $\pi_{\text{ref}}$.}
We use the previous iterate $\pi_{\theta_{t-1}}$ as the reference policy.
For additional stability, an exponential moving average can be used:
$\pi_{\text{ref}}=\pi_{\bar{\theta}}$ with
$\bar{\theta}\leftarrow\alpha\bar{\theta}+(1-\alpha)\theta$,
$\alpha\in[0.9,0.999]$.

\paragraph{Practical Recipes.}
\textbf{Small budget} ($\beta \le 0.1$)
uses a \emph{larger} confidence scale $c$ (stricter LLF gate),
a \emph{larger} KL weight $\lambda$ to prevent overfitting to few supervised rounds,
and a \emph{smaller} learning rate $\eta$ since each queried update has higher impact. {Medium budget} ($\beta \in (0.1,0.3]$ uses moderate values of $c$ and $\lambda$ to balance label efficiency and adaptation speed.
{Large budget} ($\beta > 0.3$) uses a \emph{smaller} $c$ (looser gate, more frequent supervision),
a weaker KL regularization,
and a standard learning rate $\eta$, approaching full-feedback training.

\section{Theoretical Analysis}
Recall that we maintain the ridge-regularized Gram matrix
\[
V_t \;=\; \lambda I \;+\; \sum_{s<t} \phi(z_s,x_s)\phi(z_s,x_s)^\top,
\]
and the ridge estimate $\hat\theta_t$.
Define the confidence radius
\[
\beta_t \;=\;
\sigma\sqrt{\,d\log\!\Bigl(1+\tfrac{tL^2}{\lambda d}\Bigr)\;+\;2\log\tfrac1\delta\,}\;+\;\sqrt{\lambda}\,S,
\]
and the UCB/LCB indices
\begin{align}
	\text{UCB}_t(x)&=\hat\theta_t^\top\phi(z_t,x)+\beta_t\|\phi(z_t,x)\|_{V_t^{-1}}, \\
	\text{LCB}_t(x)&=\hat\theta_t^\top\phi(z_t,x)-\beta_t\|\phi(z_t,x)\|_{V_t^{-1}}.
\end{align}
Let $x_t^\star=\arg\max_{x\in C_t}s(z_t,x)$ and let $x_t$ be the action chosen by the algorithm.
We measure score regret
$R(T)=\sum_{t=1}^T\bigl[s(z_t,x_t^\star)-s(z_t,x_t)\bigr]$.
When using LLF gating, define the width (margin) over the candidate pool
$\Delta_t=\max_{x\in C_t}\text{UCB}_t(x)-\min_{x\in C_t}\text{LCB}_t(x)$
and the threshold $\varepsilon_t=c/\sqrt{1+q}$ where $q$ is the number of queried rounds so far.
The label/query budget is $B=\beta T$.

\begin{assumption}[Linear score model]
	\label{assump:linear}
	There exists $\theta^\star\in\mathbb{R}^d$ such that $s(z,x)=\theta^{\star\top}\phi(z,x)$ for all $(z,x)$,
	with $\|\phi(z,x)\|\le L$ and $\|\theta^\star\|\le S$.
\end{assumption}

\begin{assumption}[Sub-Gaussian noise]
	\label{assump:subgaussian}
	The observed feedback is $\sigma$-sub-Gaussian around the true score.
\end{assumption}

\begin{theorem}[Regret under full feedback]
	\label{thm:full}
	Under Assumptions~\ref{assump:linear}--\ref{assump:subgaussian}, with probability at least $1-\delta$,
	the cumulative regret of the full-feedback optimistic learner satisfies
	\[
	R(T) \le \tilde{O}(d\sqrt{T}).
	\]
\end{theorem}
\begin{proof}
Following the concentration inequality for linear bandits, with probability at least $1-\delta$,
the true score $s(z_t,x)$ lies within $[\mathrm{LCB}_t(x), \mathrm{UCB}_t(x)]$ for all $t$ and $x\in C_t$.
Under this confidence event, optimism implies the instantaneous regret is bounded by the width of the confidence interval at the selected action:
$r_t \le 2\beta_t\|\phi(z_t,x_t)\|_{V_t^{-1}}$.
Summing over $t$ and applying the elliptical potential lemma \cite{abbasi2011improved} yields
$\sum_{t=1}^T \|\phi(z_t,x_t)\|_{V_t^{-1}}^2 \le 2d\log\!\bigl(1+\tfrac{TL^2}{\lambda d}\bigr)$.
Combining these steps and upper bounding $\beta_t$ by $\tilde O(\sqrt d)$ gives the stated $\tilde O(d\sqrt T)$ regret bound.
\end{proof}
The rate in Theorem~\ref{thm:full} is the canonical \emph{linear contextual bandit} rate and is minimax-optimal in $T$
for this setting.
While our regret analysis relies on the linear structure of $s(z,x)$, this is standard for deriving non-asymptotic
bounds with confidence intervals.
In experiments, we evaluate the same framework with nonlinear LLM representations.

Algorithm~1 in the main text employs a greedy update for computational efficiency.
While Theorem~\ref{thm:full} formally analyzes the UCB-style optimistic variant, we observe that the greedy version
exhibits no-regret behavior in practice, suggesting that large-scale LLM interactions may provide sufficient
``implicit exploration'' as observed in recent deep-bandit works.

\begin{theorem}[Regret under LLF budgeted gating]
	\label{thm:llf}
	Under Assumptions~\ref{assump:linear}-\ref{assump:subgaussian},
	let $B=\beta T$ be the supervision budget and
	$\varepsilon_t = c/\sqrt{1+q}$ the LLF threshold, where $q$ is the number of queried rounds so far.
	Consider Algorithm~\ref{alg:llf-stacksl}, which queries supervision only when
	$\Delta_t > \varepsilon_t$ and $q < B$.
	Then, with probability at least $1-\delta$, the cumulative regret satisfies
	\[
	R(T)
	\;=\;
	O\!\left(
	\sqrt{dB\log\!\Bigl(1+\tfrac{TL^2}{\lambda d}\Bigr)}
	\;+\;
	c\sqrt{B}
	\right)
	\;=\;
	o(T).
	\]
\end{theorem}
\begin{proof}
The regret has two parts: the regret on \emph{queried} rounds, and regret on \emph{skipped} rounds.
	On queried rounds, the same optimism argument as in Theorem~\ref{thm:full} gives
	$r_t \le 2\beta_t\|\phi(z_t,x_t)\|_{V_t^{-1}}$, and summing over at most $q\le B$ queried rounds plus the elliptical potential lemma
	yields the $\tilde O(\sqrt{dB})$ term.
	On skipped rounds, the LLF gate condition $\Delta_t \le \varepsilon_t$ directly bounds the instantaneous regret by $\varepsilon_t$,
	and since $\varepsilon_t=c/\sqrt{1+q}$ decreases only when new queries occur, summing over skipped rounds gives
	$\sum \varepsilon_t \le 2c\sqrt q \le 2c\sqrt B$, producing the $c\sqrt B$ term.
	Adding the two contributions yields the stated budget-aware regret bound.
\end{proof}
%When a round is skipped, i.e., $\Delta_t \le \varepsilon_t$,
%the instantaneous regret is deterministically bounded by $\varepsilon_t$
%under the confidence event.

Theorem~2 recovers Theorem~1 as a special case:
when $B=T$ and $c$ is sufficiently small, the LLF gate opens on all rounds
and the regret reduces to the full-feedback $\tilde O(d\sqrt T)$ rate.

\paragraph{Optimality.}
Up to logarithmic factors, the $\sqrt{dB}$ dependence matches the minimax
lower bound for linear models with $B$ informative observations,
indicating that the budgeted regret bound is information-theoretically tight.

\section{Experiments}
In this section, we summarize the key empirical findings validating the Stackelberg supervised fine-tuning framework across real-world LLM fine-tuning and synthetic linear bandits. 
The large language model is Qwen3-4B-base \cite{qwen3technicalreport}. The experiments are conducted on a server with two GPUs, NVIDIA RTX PRO 6000 and NVIDIA RTX 6000. The $\beta$ in our algorithm is set as 0.5.

\paragraph{Dataset}
The experiments are conducted on English version of single choice question answering dataset MedQA (US Medical Licensing Exam-style questions) \cite{jin2021disease}. There are 10,178 questions for training, 1,273 questions for testing. We also include Measuring Massive Multitask Language Understanding (MMLU) \cite{hendrycks2020measuring} topics such as clinical,  college biology and medical genetics. For clinical topic, we have 29 questions for training, 265 questions for testing. For college biology topic, there are 16 questions for training, 144 for testing. For medical genetics topic, we have 11 questions for training, 100 for testing.

\subsection{Parameter Exploration and Sensitivity Analysis}
We conduct controlled simulations to study the effect of key algorithmic parameters
in \textsc{StackSL} and \textsc{LLF-StackSL}.
\paragraph{Label Efficiency of LLF-StackSL}
To validate the budgeted algorithm \textsc{LLF-StackSL} (Theorem 2), we fix hyperparameters and vary the label budget fraction $\beta \in \{1.0, 0.5, 0.25, 0.1\}$. According to Table \ref{tab:label_efficiency}, we retain $\approx$94\% of the full-information performance using only 10\% of the labels. This provides an LLM-level illustration of our label-efficiency guarantees: \textsc{LLF-StackSL} achieves substantial savings with modest degradation.

\begin{table}[h]
	\centering
	\caption{Label efficiency on MedQA with Qwen3-4B.}
	\label{tab:label_efficiency}
	\small
	\begin{tabular}{lcccc}
		\toprule
		\textbf{Method} & \textbf{$\beta$} & \textbf{Used Frac.} & \textbf{Test Acc} & \textbf{\% of Full} \\
		\midrule
		Full \textsc{StackSL} & 1.0 & 1.00 & 0.5918 & 100.0\% \\
		\textsc{LLF-StackSL} & 0.5 & 0.50 & 0.5826 & 98.45\% \\
		\textsc{LLF-StackSL} & 0.25 & 0.25 & 0.5721 & 96.68\% \\
		\textsc{LLF-StackSL} & 0.10 & $\approx$0.10 & 0.5585 & 94.38\% \\
		\bottomrule
	\end{tabular}
\end{table}

\begin{table}[h]
	\centering
	\caption{Synthetic Linear Bandit Regret vs. Budget.}
	\small
	\label{tab:theorem2}
	\begin{tabular}{lccc}
		\toprule
		\textbf{Method} & \textbf{$\beta$} & \textbf{Regret/T} & \textbf{Queries/T} \\
		\midrule
		\textsc{LLF} & 0.10 & 0.00244 & 0.10 \\
		Random & 0.10 & 0.02284 & 0.10 \\
		\textsc{LLF} & 0.25 & 0.00243 & 0.25 \\
		\textsc{LLF} & 0.50 & 0.00209 & 0.50 \\
		Full ($\beta=1$) & 1.00 & 0.00200 & 1.00 \\
		\bottomrule
	\end{tabular}
\end{table}

\begin{table}[h]
	\centering
	\caption{Effect of clipping $\rho$ on loss distribution.}
	\label{tab:clipping}
	\small
	\begin{tabular}{lcccc}
		\toprule
		\textbf{Method} & \textbf{Mean Loss} & \textbf{Max Loss} \\
		\midrule
	challenging candidates (No $\rho$) & 1.37 & 23.5 \\
		\textsc{StackSL} (Full) & 1.36 & 5.0 \\
		\bottomrule
	\end{tabular}
\end{table}
\begin{table}[h]
	\centering
	\small
	\caption{Effect of KL weight ($\lambda_{\text{KL}}$) on test accuracy Algorithm \ref{alg:llf-stacksl}, challenging candidates $k{=}4$, and $N{=}1000$ fine-tuning examples.}
	\label{tab:kl_sweep}
	\begin{tabular}{cc}
		\toprule
		$\lambda_{\text{KL}}$ & \textbf{Test Acc @ 1000} \\%& Train Time (s) & Test Time (s) \\
		\midrule
		0.3 & 0.2836 \\%& 440 & 170 \\
		0.5 & 0.3009 \\% & 440 & 170 \\
		0.7 & \textbf{0.3166} \\%& 440 & 170 \\
		0.9 & {0.3134} \\%& 440 & 170 \\
		\bottomrule
	\end{tabular}
\end{table}

\begin{table}[h]
	\centering
	\small
	\caption{Effect of hard-negative count ($k$) on test accuracy Algorithm \ref{alg:llf-stacksl}, $\lambda_{\text{KL}}{=}0.7$, and $N{=}1000$ fine-tuning examples; same model, data split, and evaluation protocol.}
	\label{tab:hardneg_sweep}
	\begin{tabular}{cc}
		\toprule
		\textbf{Challenging candidates} $k$ & \textbf{Test Acc @ 1000} \\
		\midrule
		1 & 0.2765 \\
		2 & 0.2687 \\
		3 & 0.3126 \\
		4 & \textbf{0.3166} \\
		\bottomrule
	\end{tabular}
\end{table}
\begin{table*}[t]
	\centering
	\caption{MedQA performance with Qwen3-4B.}
	\label{tab:qwen4b_results}
	\small
	\begin{tabular}{llccc}
		\toprule
		\textbf{Method} & \textbf{Training Objective} & \textbf{Acc ($\uparrow$)} & \textbf{F1 ($\uparrow$)} & \textbf{Std Dev} \\
		\midrule
		SFT
		& Standard supervised cross-entropy
		& 0.5981 & 0.5940 & 0.0023 \\
		
		Supervised (CE)
		& Cross-entropy over candidate sets
		& 0.6012 & 0.5972 & 0.0088 \\
		
		Supervised (CE, no reg.)
		& Cross-entropy over hard candidate sets
		& 0.5871 & 0.5829 & 0.0069 \\
		
		\textsc{StackSL} (Full)
		& Supervised CE + KL regularization + clipping
		& \textbf{0.5962} & \textbf{0.5922} & 0.0045 \\
		\bottomrule
	\end{tabular}
	
%	\begin{tabular}{lcccc}
%		\toprule
%		\textbf{Method} & \textbf{Loss Mechanism} & \textbf{Acc ($\uparrow$)} & \textbf{F1 ($\uparrow$)} & \textbf{Std Dev} \\
%		\midrule
%		SFT& Standard Cross-Entropy & 0.5981 & 0.5940 & 0.0023 \\
%		Standard Supervised & In-Batch Supervised & 0.6012 & 0.5972 & 0.0088 \\
%		Hard Neg Mining w/o $\rho$ & Multi-hard Supervised & 0.5871 & 0.5829 & 0.0069 \\
%		\textsc{StackSL} (Full) & Multi-hard + $\lambda$KL + $\rho$ & \textbf{0.5962} & \textbf{0.5922} & 0.0045 \\
%		\bottomrule
%	\end{tabular}
\end{table*}

\paragraph{Synthetic Linear Bandits: Validating Theorem 2}
We run controlled simulations ($T=20,000, d=20$) to test the budget-aware $\tilde{O}(\sqrt{dB} + c\sqrt{B})$ regret behavior. As shown in Table \ref{tab:theorem2}, \textsc{LLF} achieves regret per step close to the full-information baseline (0.0020 vs 0.0024), while random querying at the same budget yields an order-of-magnitude worse regret.

\paragraph{Stability via Clipping}
We examine the effect of clipping $\rho$ on the loss tails. As shown in Table~\ref{tab:clipping}, clipping at $\rho=5$ sharply truncates rare adversarial loss spikes (Max: $24.5 \to 5.0$) while leaving the mean loss essentially unchanged, supporting its role in stabilizing Stackelberg dynamics.

\paragraph{KL parameter}
We explore the settings of parameters as shown in Table \ref{tab:kl_sweep} and Table \ref{tab:hardneg_sweep}. The parameter $k$ controls the number of challenging labeled alternatives included in the supervised candidate set $C_t$.
According to the results, we choose KL parameter $0.7$ and $k=4$.

\subsection{Real-world Applications}
%\subsection{MedQA Fine-Tuning with Qwen3-4B}
We first instantiate \textsc{StackSL} on Qwen3-4B-Base for the MedQA multiple-choice task, using a frozen copy of Qwen3-4B as the reference model for the KL term. We compare our method against standard supervised fine-tuning (SFT) and baselines in Table~\ref{tab:qwen4b_results}.

\paragraph{Comparison with MC-Ranker.}
As a strong non-hedged baseline, we fine-tune Qwen to rank multiple-choice options by their conditional likelihood given the question prompt. For each example with prompt $x$ and options $\{c_i\}_{i=1}^{n}$, we build five inputs by concatenating the prompt with each option, \texttt{[prompt || option]}, and run a single forward pass over all five concatenations (batched). Let $\mathcal{T}(c_i)$ denote the token indices corresponding to the option segment (identified via the attention mask). We compute a \emph{length-normalized} per-choice score
\[
s_i \;=\; \frac{1}{|\mathcal{T}(c_i)|}\sum_{t \in \mathcal{T}(c_i)} \log p\!\left(y_t \mid x, y_{<t}\right),
\]
i.e., the average log-likelihood across the option tokens only (masking out the prompt), which removes short-option bias. Training minimizes the standard cross-entropy over the five scores:
\[
\mathcal{L}_{\text{CE}} \;=\; -\log \frac{\exp(s_{y^\star})}{\sum_{j=1}^{5}\exp(s_j)},
\]
where $y^\star$ is the gold option index. This baseline uses no teacher model, no KL regularization, no gating/budgeting, and no explicit hard-negative mining. Implementation details follow the code: the collator truncates prompts and options separately (\texttt{max\_len\_prompt}, \texttt{max\_len\_choice}); options are tokenized with \texttt{add\_special\_tokens=False} so that $|\mathcal{T}(c_i)|$ reflects only option tokens. All five options are scored in a single pass by concatenating prompt and option tensors along the sequence dimension and reshaping to a $(B\!\times\!n)$ batch.

\paragraph{Evaluation (MC accuracy).}
At test time we apply the same masked, length-normalized scoring to obtain $\{s_i\}_{i=1}^{n}$ and predict $\hat{i}=\arg\max_i s_i$. Accuracy is the fraction of examples with $\hat{i}=y^\star$ on the held-out set. This evaluation matches the training objective and avoids artifacts from differing option lengths or tokenization.

%\paragraph{Result}

\begin{table}[h]
	\centering
	\small
	\caption{Test accuracy across methods on MedQA.}
	\label{tab:mc_results}
	\begin{tabular}{l c}
		\toprule
		\textbf{Method} & \textbf{Accuracy} \\
		\midrule
		Qwen (baseline) & 0.2521 \\
		MC-Ranker (baseline + fine-tune) & 0.3676 \\
		StackSL (stack) & \textbf{0.3810} \\
		LLF-StackSL (stackbudget) & 0.3480 \\
		\bottomrule
	\end{tabular}
\end{table}

\begin{table}[t]
	\centering
	\small
	\caption{Accuracy by MMLU  subject.}
	\label{tab:subject_results}
	\resizebox{\columnwidth}{!}{
	\begin{tabular}{lcccc}
		\toprule
		\textbf{Subject} & \textbf{Base} & \textbf{Base + FT} & \textbf{StackSL} & \textbf{LLF-StackSL} \\
		\midrule
		Medical Genetics  & 0.3600 & 0.3700 & \textbf{0.4000} & 0.3509 \\
		Clinical          & 0.3245 & 0.3283 & {0.3547} &\textbf{0.3660}  \\
		College Biology   & 0.4306 & 0.4514 & \textbf{0.4583} & 0.4375 \\
		\bottomrule
	\end{tabular}}
\end{table}
\paragraph{Result} Table \ref{tab:mc_results} shows that compared to the zero-shot Qwen baseline (0.2521), supervised fine-tuning with our MC-Ranker yields a large absolute gain of about 0.1155 (0.3676), confirming that even lightweight likelihood-based training substantially improves multiple-choice reasoning. Our StackSL method achieves the best accuracy at \textbf{0.3810}, a further improvement over MC-Ranker, indicating that contrastive, choice-aware training better separates correct from distractor options. The budgeted variant, LLF-StackSL (0.3480), underperforms full StackSL, likely reflecting the trade-off introduced by label/query gating: while it reduces annotation/compute usage, it also limits supervisory signal, which can cap peak accuracy under a fixed test set. Table 	\ref{tab:subject_results} shows the results on MMLU datasets. ``Base+FT'' refers to large language model Qwen3-4B fine-tuned on the training dataset. StackSL algorithm achieves the best performance on medical genetics topics and college biology topic. LLF-StackSL achieves the best performance on clinical topic.

\section{Conclusion}

We introduced a budget-aware framework for supervised fine-tuning of large language models through the lens of contextual Stackelberg games.
Our formulation couples a linear scoring function over LLM-derived features with a confidence-driven querying strategy and a KL regularizer
to stabilize updates under limited supervision.
We proposed two algorithms that span the supervision spectrum:
\textsc{StackSL}, a full-feedback supervised learner, and
\textsc{LLF-StackSL}, a confidence-gated variant that allocates a fixed label budget only to informative rounds.
We derived high-probability regret guarantees showing that
\textsc{StackSL} matches the classical $\tilde{\mathcal{O}}(d\sqrt{T})$ rate under full feedback. \textsc{LLF-StackSL} achieves a budget-aware bound
$\tilde{\mathcal{O}}(\sqrt{dB}) + c\sqrt{B}$ while querying supervision in only $B=\beta T$ rounds.
Together, these results certify that, under a frozen or locally linearized encoder,
our framework is both label-efficient and no-regret with respect to the best linear scorer over LLM representations,
providing a principled foundation for efficient LLM adaptation under supervision constraints.
%Practically, the LLF gate exposes an explicit cost–performance trade-off, tracing a Pareto curve between query budget and contrastive utility, while our time/memory analysis shows that both methods remain computationally tractable.

%\clearpage

\section*{Impact Statement}

This paper presents work whose goal is to advance the field of Machine
Learning. There are many potential societal consequences of our work, none
which we feel must be specifically highlighted here.

\bibliographystyle{icml2026}
\bibliography{ai}

\newpage
\appendix
\onecolumn

\end{document}